\documentclass{llncs}

\usepackage{times}
\usepackage{array}
\usepackage{url}

  \usepackage[vlined,ruled,linesnumbered]{algorithm2e}

  \usepackage{latexsym}
  \usepackage{xspace}
  \usepackage{amsmath,amssymb,amsfonts}
  \usepackage{mathrsfs}
  \usepackage{stmaryrd}
  \usepackage{graphicx}

\newcommand{\TRONLY}[1]{#1}
\newcommand{\ARTONLY}[1]{}

  \newcommand{\hide}[1]{}

\newcommand{\nent}{\ensuremath{\mathrel{\mbox{$|\hspace*{-.48em}\approx$}}}\xspace}
\newcommand{\pre}{\mathsf{pre}}
\newcommand{\con}{\mathsf{con}}
\newcommand{\Sig}{\ensuremath{\mathit{Sig}}\xspace}
\newcommand{\sig}{\ensuremath{\mathrm{sig}}\xspace}
\newcommand{\Mod}{\ensuremath{\mathrm{Mod}}\xspace}
\newcommand{\Modd}{\ensuremath{\mathrm{Mod_{DI}}}\xspace}

\newcommand{\rank}{\mathit{rank(\cdot)}}

\newcommand{\tbs}{{\small\ensuremath{\top\bot^*}}\xspace}
\newcommand{\modex}{\textsf{\small Mod}\xspace}
\newcommand{\opt}{\textsf{\small Opt}\xspace}

  \newcommand{\D}{\ensuremath{\mathcal{D}}\xspace}
  \newcommand{\E}{\ensuremath{\mathcal{E}}\xspace}

  \newcommand{\M}{\ensuremath{\mathcal{M}}\xspace}
  \newcommand{\N}{\ensuremath{\mathrm{N}}\xspace}

  \renewcommand{\S}{\ensuremath{\mathcal{S}}\xspace}
  \newcommand{\T}{\ensuremath{\mathcal{T}}\xspace}

  \newcommand{\DL}{\ensuremath{\mathcal{DL}}\xspace}
  \newcommand{\DLN}{\ensuremath{\mathcal{DL}^\N}\xspace}
  
  \newcommand{\ELpp}{\ensuremath{\mathcal{EL}^{++}}\xspace}

  \newcommand{\KB}{\ensuremath{\mathcal{KB}}\xspace}
  \newcommand{\KBall}{\ensuremath{\mathcal{KB}_\mathit{all}}\xspace}
  \newcommand{\DLL}{\textit{DL-lite}\xspace}

  \newcommand{\tmS}{\ensuremath{\mathcal{S}}\xspace}

 \newcommand{\jo}{\mathtt{J}}
	\newcommand{\gl}{\mathtt{G}}
	\newcommand{\pn}{\mathtt{P}}
	\newcommand{\is}{\mathtt{M}}

{\qed \end{trivlist}}

\newenvironment{theorem*}[2]%
{\begin{trivlist} \item[] {\bf #1~\protect{\ref{#2}}}\it}{\end{trivlist}}

  {

      \begin{enumerate}
  }%
  {
      \end{enumerate}
  }

  {

      \begin{enumerate}
  }%
  {
      \end{enumerate}
  }

\begin{document}

\title{Optimizing the computation of overriding}

\author{P.A.~Bonatti \and I.M.~Petrova \and L.~Sauro}

\institute{Dip.\ Ing.\ Elet.\ e Tecnologie dell'Informazione, Universit\`a di Napoli Federico II}

\maketitle

  \begin{abstract}
We introduce optimization techniques for reasoning in \DLN---a recently introduced family of nonmonotonic description logics whose characterizing features appear well-suited to model the applicative examples naturally arising in biomedical domains and semantic web access control policies. Such optimizations are validated experimentally on large KBs with more than 30K axioms. Speedups exceed 1 order of magnitude.
For the first time, response times compatible with real-time reasoning are obtained with nonmonotonic KBs of this size.
  \end{abstract}

  \section{Introduction}

Recently, a new family of nonmonotonic Description Logics (DLs),
called \DLN, has been introduced \cite{DLN-15}. It supports
\emph{normality concepts} $\N C$ to denote the normal/standard/\break
prototypical instances of a concept $C$, and prioritized
\emph{defeasible inclusions} (DIs)\break $C\sqsubseteq_n D$ with the following 
meaning: ``\emph{by default, the instances of $C$ satisfy
  $D$, unless stated otherwise}'', that is, unless some higher
priority axioms entail $C \sqcap \neg D$; in that case,
$C\sqsubseteq_n D$ is \emph{overridden}. The normal/standard/prototypical
instances of $C$ are required to satisfy all the DIs that are not
overridden in $C$.

Given the negligible number of applications based on nonmonotonic
logics deployed so far, \DLN has been designed to address real-world
problems and concrete knowledge engineering needs. In this regard, the
literature provides  clear and
articulated discussions of how nonmonotonic reasoning can be of help
in important contexts related to the semantic web, such as biomedical
ontologies
\cite{DBLP:conf/psb/Rector04,DBLP:journals/ijmms/StevensAWSDHR07}
(with several applications, such as literature search) and (semantic
web) policy formulation \cite{DBLP:journals/csec/WooL93}. 
These and other applications are extensively discussed in \cite{DLN-15}.

The distinguishing features in \DLN's design are: ($i$) \DLN adopts
the simplest possible criterion for overriding, that is, inconsistency
with higher priority axioms; ($ii$) all the normal instances of a
concept $C$ conform to the same set of default properties, sometimes
called \emph{prototype} in the following; ($iii$) the conflicts
between DIs that cannot be resolved with priorities are regarded as
knowledge representation errors and are to be fixed by the knowledge
engineer (typically, by adding specific DIs).  No traditional
nonmonotonic logic satisfies ($i$), and very few satisfy ($ii$) or
($iii$). \DLN behaves very well on applicative examples due to the
following consequences of ($i$)--($iii$) (a comparison with other
nonmonotonic DLs with respect to these features is summarized in
Table~\ref{summary-of-comparisons}):

\emph{No inheritance blocking}: In several nonmonotonic logics a concept with exceptional properties inherits \emph{none} of the default properties of its superclasses. This undesirable phenomenon is known as \emph{inheritance blocking}.

\emph{No undesired closed-world assumption (CWA) effects}: In some nonmonotonic DLs, an exceptional concept is shrinked to the individuals that explicitly belong to it, if any; hence, it may become inconsistent.

\emph{Control on role ranges}:  Unlike most nonmonotonic DLs, \DLN axioms can specify whether a role should range only over normal individuals or not.

\emph{Detect inconsistent prototypes}: \DLN facilitates the
identification of all conflicts that cannot be resolved with
priorities (via consistency checks over normality concepts), because
their correct resolution is application dependent and should require
human intervention (cf.\ \cite[Sec.~1]{DLN-15} and
Example~\ref{ex:policy-conflict} below).

\emph{Tractability}: \DLN is currently the only nonmonotonic DL known to preserve the tractability of all low-complexity DLs, including \ELpp and \DLL (that underly the OWL2-EL and OWL2-QL profiles). This opens the way to processing very large nonmonotonic KBs within these fragments.

\begin{table}
  \label{summary-of-comparisons}
  \footnotesize
  \vspace*{-1em}
  \begin{center}
    \begin{tabular}{|m{15em}||>{\centering\arraybackslash}m{2.1em}|>{\centering\arraybackslash}m{2.1em}|>{\centering\arraybackslash}m{2.1em}|>{\centering\arraybackslash}m{4.2em}|>{\centering\arraybackslash}m{1.8em}|>{\centering\arraybackslash}m{1.8em}|>{\centering\arraybackslash}m{1.8em}|>{\centering\arraybackslash}m{2.2em}|}
      \hline
      \multicolumn{1}{|p{8.5em}||}{} &
      \multicolumn{1}{c|}{CIRC} &
      \multicolumn{1}{c|}{DEF} &
      \multicolumn{1}{c|}{AEL} &
      \multicolumn{1}{c|}{TYP} &
      \multicolumn{2}{c|}{RAT} &
      \multicolumn{1}{c|}{PR} &
      \\
      \textbf{Features}
      & \footnotesize 
        \cite{DBLP:journals/jair/BonattiLW09,DBLP:journals/jair/BonattiFS11}
      & \footnotesize 
        \cite{DBLP:journals/jar/BaaderH95,DBLP:journals/jar/BaaderH95a}
      & \footnotesize 
        \cite{DBLP:journals/tocl/DoniniNR02}
      & \footnotesize 
        \cite{DBLP:journals/fuin/GiordanoOGP09,DBLP:journals/ai/GiordanoGOP13}
      & 
        \multicolumn{1}{c}{
        \cite{DBLP:conf/jelia/CasiniS10,DBLP:conf/dlog/CasiniMMV13}
        }
      & \footnotesize
        \cite{DBLP:journals/jair/CasiniS13}
      & \footnotesize 
        \cite{DBLP:journals/ai/Lukasiewicz08} 
      & \footnotesize 
        \DLN 
      \\
      \hline
      \hline
      no inheritance blocking
      & \checkmark 
      & \checkmark 
      & \checkmark 
      &  
      &  
      & \checkmark 
      & \checkmark 
      & \checkmark 
      \\
      \hline
      no CWA effects
      &  
      & \checkmark 
      & \checkmark 
      &  
      & \checkmark 
      & \checkmark 
      &  
      & \checkmark 
      \\
      \hline
      fine-grained control on role ranges
      &  
      &  
      &  
      &  sometimes 
      &  
      &  
      &  
      & \checkmark 
      \\
      \hline
      detects inconsistent prototypes
      &  
      &  
      &  
      & sometimes 
      &  
      &  
      & \checkmark 
      & \checkmark 
      \\
\hide{
      \hline
      unique deductive~~~\break closure
      & \checkmark 
      &  
      &  
      & \checkmark 
      &  
      &  
      &  
      & \checkmark 
      \\
}
\hide{      \hline
      preserves legacy~~~\break taxonomies
      &  
      &  
      &  
      &  
      &  
      &  
      &  
      & \checkmark 
      \\
}
      \hline
      preserves tractability
      &  
      &  
      &  
      &  
      &  
      &  
      &  
      & ~\checkmark$^{(*)^{~}}$   
      \\
\hide{
      \hline
      \hline
      implicit specificity
      &  
      &  
      &  
      & \checkmark 
      & \checkmark 
      & \checkmark 
      & \checkmark 
      &  
      \\
      \hline
      other priorities
      & \checkmark 
      & \checkmark 
      &  
      &  
      &  
      &  
      &  
      & \checkmark 
      \\
}
      \hline
      \multicolumn{9}{p{36em}}{\footnotesize (*)\,  
        It holds for subsumption, assertion checking, concept consistency, KB consistency.}\\
    \end{tabular}
  \end{center}
  \vspace*{-0.5em}
\caption{Partial comparison with other nonmonotonic DLs, cf.\cite{DLN-15}, where CIRC, DEF, AEL, TYP, RAT, PR stand, respectively, for Circumscribed DLs, Default DLs, Autoepistemic DLs, DLs with Typicality, DLs with Rational Closure, and Probabilistic DLs.}
  \vspace*{-2em}
\end{table}

The performance of \DLN inference has been experimentally analyzed on
large KBs (with more than 20K concept names and over 30K
inclusions). The results are promising; still, as defeasible
inclusions approach 25\% of the KB, query response time slows down enough to call for improvements.
In this paper, we study two optimization techniques to improve \DLN query response time:
\begin{enumerate}
\item Many of the axioms in a large KB are expected to be irrelevant to the given query.  We investigate the use of \emph{module extractors} \cite{DBLP:conf/dlog/SattlerSZ09,MaWa-DL14} to focus reasoning on relevant axioms only. The approach is not trivial (module extractors are unsound for most nonmonotonic logics, including circumscription, default and autoepistemic logics) and requires an articulated correctness proof.

\item We introduce a new algorithm for query answering, that is expected to exploit incremental reasoners at their best. Incremental reasoning is crucial as \DLN's reasoning method iterates consistency tests on a set of KBs with large intersections. While the assertion of new axioms is processed very efficiently, the  computational cost of axiom deletion is generally not negligible. We introduce an \emph{optimistic reasoning method} that is expected to reduce the number of deletions.
\end{enumerate}

\noindent
Both optimizations are validated experimentally.  Speedups  exceed 1 order of magnitude.
To the best of our knowledge, this is the first time that response times compatible with real-time reasoning are obtained with nonmonotonic KBs of this size.

The paper is organized as follows: Sec.~\ref{sec:prelim} provides the
basics of \DLN and illustrates its inferences with examples.  Sections~\ref{sec:module-extractor} and
\ref{sec:optimistic-method} introduce the two optimization methods, respectively, and
prove their correctness. Their experimental assessment is in
Sec.~\ref{sec:experiments}.
\ARTONLY{Proofs have been omitted due to space limitations. They can be found in \cite{??}, together with further explanations and examples.}
We assume the reader to be familiar with description logics, see
\cite{BaaCa10} for all details.
The code and test suites are available at: \url{http://goo.gl/KnMO9l}.

  \section{Preliminaries}
  \label{sec:prelim}

Let \DL be any classical description logic language (see
\cite{BaaCa10} for definitions), and let \DLN be the extension of
\DL with a new concept name $\N C$ for each \DL concept $C$. The new
concepts are called \emph{normality concepts}.

A \DLN \emph{knowledge base} is a disjoint union $\KB=\tmS \cup \D$
where \tmS is a finite set of \DLN inclusions and assertions (called
\emph{strong} or classical axioms) and \D is a finite set of
\emph{defeasible inclusions} (DIs, for short) that are expressions
$C\sqsubseteq_n D$ where $C$ is a \DL concept and $D$ a \DLN
concept.
If $\delta=(C\sqsubseteq_n D)$, then $\pre(\delta)$ and $\con(\delta)$ denote $C$ and $D$, respectively.
Informally speaking, the set of DIs satisfied by all the instances of a normality concept $\N C$ constitute the \emph{prototype} associated to $C$.

DIs are prioritized by a strict partial order $\prec$. If
$\delta_1\prec\delta_2$, then $\delta_1$ has higher priority than
$\delta_2$. \DLN solves automatically only the conflicts that can be
settled using $\prec$; any other conflict shall be resolved by the
knowledge engineer (typically by adding suitable DIs).
Two priority relations have been investigated so far. Both
are based on \emph{specificity}: the specific default properties of a
concept $C$ have higher priority than the more generic properties of
its superconcepts (i.e.\ those that subsume $C$). The priority
relation used in most of \cite{DLN-15}'s examples identifies those
superconcepts with strong axioms only:
\begin{equation}
  \label{specificity}
  \delta_1\prec\delta_2 \mbox{ iff }
  \pre(\delta_1)\sqsubseteq_\tmS \pre(\delta_2) \mbox{ and }
  \pre(\delta_2)\not\sqsubseteq_\tmS \pre(\delta_1) \,.\footnote{As usual, $C\sqsubseteq_\tmS D$ means that $\tmS \models C\sqsubseteq D$.}
\end{equation}
The second priority relation investigated in \cite{DLN-15} is
\begin{equation}
  \label{specificity-2}
  \delta_1 \prec \delta_2 \mbox{ iff } \mathit{rank(\delta_1) >
    rank(\delta_2)},
\end{equation}
where $\mathit{rank(\cdot)}$ is shown in Algorithm \ref{alg:rank} and corresponds to 
the ranking function of rational closure \cite{DBLP:conf/jelia/CasiniS10,DBLP:journals/jair/CasiniS13}.
This relation uses also DIs to determine superconcepts, so (roughly
speaking) a DI $C \sqsubseteq_n D$---besides defining a default
property for $C$---gives the specific default properties of $C$ higher
priority than those of $D$. The advantage of this priority relation is
that it resolves more conflicts than (\ref{specificity}); the main
advantage of (\ref{specificity}) is predictability; e.g.\ the effects
of adding default properties to an existing, classical KB are more
predictable, as the hierarchy used for determining specificity and
resolving conflicts is the original, validated one, and is not
affected by the new DIs (see also the related discussion in
\cite{DBLP:conf/semweb/BonattiFS10,DBLP:conf/aaai/BonattiFS11}, that
adopt (\ref{specificity})).

\begin{algorithm}[h]
  \label{alg:rank}
  \caption{Ranking function}
  \small
  \dontprintsemicolon
  \KwIn{Ontology $\KB=\S\cup \D$}
  \KwOut{the function $\rank$}
  \BlankLine
      $i:=-1$;\quad
      $\E_0 := \{C\sqsubseteq D \mid C\sqsubseteq_n D\in\D\}$ \;
      \Repeat{$\E_{i+1}=\E_i$}{ 
    	   $i:=i+1$\;
        $\E_{i+1}:= \{C\sqsubseteq D \in \E_i\mid \S\cup \E_{i}\models C\sqsubseteq \perp\}$\;
	      \ForAll{$C\sqsubseteq_n D\mbox{ s.t. }C\sqsubseteq D\in \E_i\setminus \E_{i+1}$}{
            assign $\mathit{rank}(C\sqsubseteq_n D):=i$\;
	       }
      }
      \textbf{forall} $C\sqsubseteq_n D\in\E_{i+1}$ \textbf{do} assign $\mathit{rank}(C\sqsubseteq_n D):=\infty$ \;
      \Return $\rank$\; 
\end{algorithm}

The expression $\KB \nent \alpha$ means that $\alpha$ is a \DLN
\emph{consequence} of \KB. Due to space limitations, we do not report
the model-theoretic definition of \nent and present only its reduction to classical reasoning \cite{DLN-15}. For all subsumptions and assertions $\alpha$, $\KB \nent \alpha$
holds iff $\KB^\Sigma \models \alpha$, where\hide{$\models$ is classical
entailment,} $\Sigma$ is the set of normality concepts that explicitly occur in
$\KB\cup\{\alpha\}$, and $\KB^\Sigma$ is a classical knowledge base
obtained as follows (recall that $\KB=\tmS\cup\D$):

First, for all DIs $\delta\in\D$ and all $\N C\in\Sigma$, let:
\begin{equation}
  \label{DI-trans}
  \delta^{\N C} = \big(\N C \sqcap \pre(\delta) \sqsubseteq \con(\delta)\big) \,.
\end{equation}
The informal meaning of $\delta^{\N C}$ is: ``$\N C$'s instances satisfy $\delta$''.

Second, let $\tmS'\downarrow_{\prec \delta}$ denote the result of removing from the axiom set $\tmS'$ all the $\delta_0^{\N C}$ such that $\delta_0 \not\prec \delta$:
\[
\tmS'\downarrow_{\prec \delta} = \tmS' \setminus
     \{ \delta_0^{\N C} \mid \N C\in \Sigma \land
        \delta_0 \not\prec \delta \} \,.
\]

Third, let $\delta_1,\ldots,\delta_{|\D|}$ be any \emph{linearization} of
$(\D,\prec)$.\footnote{That is, $\{\delta_1,\ldots,\delta_{|\D|}\}=\D$ and for all
$i,j=1,\ldots,{|\D|}$, if $\delta_i \prec \delta_j$ then $i<j$. 
}

Finally, let $\KB^\Sigma = \KB^\Sigma_{|\D|}$, where the sequence $\KB^\Sigma_i$ ($i=1,2,\ldots,{|\D|}$) is inductively defined as follows:
\begin{eqnarray}
  \label{KB-Sigma-constr-1}
  \KB^\Sigma_0 &=& \tmS \cup \big\{\N C\sqsubseteq C \mid \N C\in\Sigma\big\} 
  \\
  \nonumber
  \KB^\Sigma_i &=& \KB^\Sigma_{i-1} \cup \big\{\delta_i^{\N C} \mid \delta_i\in\KB, 
       \N C\in\Sigma,
       \mbox{ and } 
       \\
       \label{KB-Sigma-constr-2}
       & & \KB^\Sigma_{i-1}\downarrow_{\prec \delta_i} \cup\ \{\delta_i^{\N C}\} \not\models
       \N C\sqsubseteq \bot \big\}\,.
\end{eqnarray}
In other words, the above sequence starts with \KB's strong axioms
extended with the inclusions $\N C\sqsubseteq C$, then processes the
DIs $\delta_i$ in non-increasing priority order. If $\delta_i$ can be
consistently added to $C$'s prototype, given all higher priority DIs
selected so far (which is verified by checking that $\N
C\not\sqsubseteq \bot$ in line (\ref{KB-Sigma-constr-2})), then its
translation $\delta_i^{\N C}$ is included in $\KB^\Sigma$
(i.e.\ $\delta_i$ enters $C$'s prototype), otherwise $\delta_i$ is
discarded, and we say that $\delta_i$ is \emph{overridden in $\N C$}.

     \subsection{Examples}
     \label{sec:examples}

We start with a brief discussion of \DLN's conflict handling. Most
other logics silently neutralize the conflicts between nonmonotonic
axioms with the same (or incomparable) priorities by computing the
inferences that are invariant across all possible ways of resolving
the conflict.
A knowledge engineer
might solve it in favor of \emph{some} of its possible resolutions, instead;
however, if the logic silently neutralizes the conflict, then missing knowledge may remain undetected and unfixed. 
This approach may cause
serious problems in the policy domain:

\begin{example}
  \label{ex:policy-conflict}
  Suppose that project coordinators are both administrative staff and
  research staff. By default, administrative staff are allowed to sign
  payments, while research staff are not. A conflict arises since both
  of these default policies apply to project coordinators. Formally,
  \KB can be formalized with: \\ {\small
  \begin{minipage}{.48\textwidth}
    \begin{eqnarray}
      \label{pol-conf-1}
      & \mathtt{Admin \sqsubseteq_n \exists has\_right.Sign}\\
      \label{pol-conf-2}
      & \mathtt{Research \sqsubseteq_n \neg \exists has\_right.Sign}
    \end{eqnarray}~
  \end{minipage}
  ~
  \begin{minipage}{.48\textwidth}
    \begin{eqnarray}
      \label{pol-conf-3}
      & \mathtt{PrjCrd \sqsubseteq Admin \sqcap Research }
    \end{eqnarray}~
  \end{minipage}
  }

  \noindent
  Leaving the conflict unresolved may cause a variety of security
  problems. If project coordinators should \emph{not} sign payments,
  and the default policy is \emph{open} (authorizations are granted by
  default), 
  then failing to infer $\mathtt{\neg \exists
    has\_right.Sign}$ would\hide{may} improperly authorize the signing
  operation. Conversely, if the authorization is to be granted, then
  failing to prove $\mathtt{ \exists has\_right.Sign}$ causes a
  \emph{denial of service} (the user is unable to complete a legal
  operation). To prevent these problems, \DLN makes the conflict
  visible by inferring $\KB\nent \N\,\mathtt{PrjCrd} \sqsubseteq
  \bot$ (showing that $\mathtt{PrjCrd}$'s prototype is
  inconsistent).
  This can be proved by checking that
  $\KB^\Sigma \models \N\,\mathtt{PrjCrd} \sqsubseteq
  \bot$, where $\Sigma=\{\N\,\mathtt{PrjCrd}\}$. Then $\KB^\Sigma$ consists of  (\ref{pol-conf-3}), $\N\,\mathtt{PrjCrd} \sqsubseteq \mathtt{PrjCrd}$, and the translation of (\ref{pol-conf-1}) \emph{and} (\ref{pol-conf-2}) (none overrides the other because none is more specific under any of the two priorities):{\small
    $$
    \begin{array}{l}
      \N\,\mathtt{PrjCrd} \sqcap \mathtt{Admin \sqsubseteq \exists has\_right.Sign},\\  \N\,\mathtt{PrjCrd} \sqcap \mathtt{Research \sqsubseteq \neg \exists has\_right.Sign}. \vspace*{-3em} 
    \end{array}
    $$}
\qed
\end{example}

\noindent
Here is another applicative example from the semantic policy domain,
showing \DLN's behavior on multiple exception levels.

\begin{example}
  \label{policy-0}
  We are going to axiomatize the following natural language policy:
  \emph{``In general, users cannot access confidential files; Staff
    can read confidential files; Blacklisted users are not granted
    any access. This directive cannot be overridden.''}  Note that
  each of the above directives contradicts (and is supposed to
  override) its predecessor in some particular case.  Authorizations
  can be reified as objects with attributes \emph{subject} (the access
  requestor), \emph{target} (the file to be accessed), and
  \emph{privilege} (such as \emph{read} and \emph{write}).  Then the
  above policy can be encoded as follows:
  {\small
  \begin{eqnarray}
    \label{pol-0-0}
    \mathtt{Staff} & \sqsubseteq & \mathtt{User}
    \\
    \label{pol-0-1}
    \mathtt{Blklst} & \sqsubseteq & \mathtt{Staff}
    \\
    \label{pol-0-2}
    \mathtt{UserReqst}
    & \sqsubseteq_n & 
    \mathtt{\neg\exists privilege}
    \\
    \label{pol-0-3}
    \mathtt{StaffReqst}
    & \sqsubseteq_n & 
    \mathtt{~~\,\exists privilege.Read}
    ~~~~~~~~~
    \\
    \label{pol-0-4}
    \mathtt{BlkReq}
    & \sqsubseteq & 
    \mathtt{\neg\exists privilege}
  \end{eqnarray}%
  }%
  where $\mathtt{BlkReq} \doteq \mathtt{\exists subj.Blklst}$, $\mathtt{StaffReqst} \doteq
  \mathtt{\exists subj.Staff}$, and
  $\mathtt{UserReqst} \doteq \mathtt{\exists subj.User}$.
  By (\ref{pol-0-0}), both the specifity relations (\ref{specificity}) and (\ref{specificity-2}) yield
  $(\ref{pol-0-3})\prec(\ref{pol-0-2})$, that is, (\ref{pol-0-3}) has
  higher priority than (\ref{pol-0-2}).  Let $\Sigma=\{\N \mathtt{StaffReqst}\}$; (\ref{pol-0-3}) overrides (\ref{pol-0-2}) in $\N \mathtt{StaffReqst}$ (under (\ref{specificity}) as well as (\ref{specificity-2})), so
  $\KB^\Sigma$ consists of: (\ref{pol-0-0}),  (\ref{pol-0-1}),  (\ref{pol-0-4}), plus
{\small
  \begin{eqnarray*}
    \mathtt{\N \mathtt{StaffReqst} }
    & \sqsubseteq &  \mathtt{StaffReqst}
    \\
    \mathtt{\N \mathtt{StaffReqst} \sqcap \mathtt{StaffReqst} }
    & \sqsubseteq & 
    \mathtt{~~\,\exists privilege.Read}
    ~~~~~~~~~
  \end{eqnarray*}
}
  Consequently, $\KB\nent \mathtt{\N\mathtt{StaffReqst}  \sqsubseteq
      \exists privilege.Read}$. Similarly, it can be verified that:
  \begin{enumerate}
  \item Normally, access requests involving confidential files are
    rejected, if they come from generic users: $\KB\nent \mathtt{\N\mathtt{UserReqst}  \sqsubseteq \neg\exists privilege}$;
  \item Blacklisted users cannot do anything by (\ref{pol-0-4}), so, in
    particular:\\ $\KB\nent \mathtt{\N\mathtt{BlkReq} \sqsubseteq
      \neg\exists privilege}$.\qed
  \end{enumerate}
\end{example}

\noindent
Some applicative examples from the biomedical domain can be found in
\cite{DLN-15} (see Examples~3, 4, 10, 12, and the drug contraindication
example in Appendix C). Like the above examples, they are all correctly
solved by \DLN with both priority notions. Applicative examples hardly exhibit the complicated networks of dependencies between conflicting defaults that occur in artificial examples. Nonetheless, we briefly discuss the artificial examples, too, as a mean of comparing \DLN with other logics such as \cite{DBLP:journals/ai/Sandewall10,DBLP:journals/jair/CasiniS13,DBLP:journals/jair/BonattiFS11}. 

In several cases, e.g.\ examples B.4 and B.5 in \cite{DBLP:journals/ai/Sandewall10}, \DLN agrees with \cite{DBLP:journals/ai/Sandewall10,DBLP:journals/jair/CasiniS13,DBLP:journals/jair/BonattiFS11} under both priority relations. 
Due to space limitations, we illustrate only B.4.
 
\begin{example}[Juvenile offender]
Let $\KB$ consist of axioms (\ref{axiom:KBfirst})--(\ref{axiom:j2}) where $\jo$, $\gl$, $\is$, $\pn$ abbreviate $\mathtt{JuvenileOffender}$, $\mathtt{GuiltyOfCrime}$, $\mathtt{IsMinor}$ and $\mathtt{ToBePunished}$, respectively.    

\noindent
{\small
\begin{minipage}{0.5\textwidth}
\begin{eqnarray}
\label{axiom:KBfirst} \jo &\sqsubseteq & \gl 
\\
\jo &\sqsubseteq & \is	
\\
\label{axiom:j0}  \is\sqcap \gl &\sqsubseteq_n & \neg 	\pn \
\\
\label{axiom:j1} \is &\sqsubseteq_n & \neg \pn 
\\
\label{axiom:j2} \gl &\sqsubseteq_n & \pn 
\end{eqnarray}
\end{minipage}
\begin{minipage}{0.5\textwidth}
\begin{eqnarray}
\label{axiom:KBSf} \jo &\sqsubseteq & \gl\\
\jo &\sqsubseteq & \is\\
\N\jo &\sqsubseteq  & \jo\\
\label{axiom:j3} \N\jo\sqcap \is\sqcap \gl &\sqsubseteq & \neg \pn\\
\label{axiom:KBSl} \N\jo\sqcap \is &\sqsubseteq  & \neg \pn
\end{eqnarray}
\end{minipage}
}

\vspace{10pt}
\noindent
On one hand, criminals have to be punished and, on the other hand,
minors cannot be punished.  So, what about juvenile offenders?
The defeasible inclusion (\ref{axiom:j0}) breaks the tie in favor of
their being underage, hence not punishable.  By setting
$\Sigma=\{\N\jo\}$, priorities (\ref{specificity}) and (\ref{specificity-2}) both
return axioms (\ref{axiom:KBSf})--(\ref{axiom:KBSl}) as
$\KB^{\Sigma}$.  Then, clearly, $\KB^{\Sigma}\models
\N\jo\sqsubseteq \neg \pn$ which is \DLN's analogue of the
inferences of
\cite{DBLP:journals/ai/Sandewall10,DBLP:journals/jair/CasiniS13,DBLP:journals/jair/BonattiFS11}.
\end{example}

In other cases (e.g.\ example B.1 in  \cite{DBLP:journals/ai/Sandewall10}) \DLN finds the same conflicts as \cite{DBLP:journals/ai/Sandewall10,DBLP:journals/jair/CasiniS13,DBLP:journals/jair/BonattiFS11}. 
However, \DLN's semantics signals these conflicts 
to the knowledge engineer whereas in \cite{DBLP:journals/ai/Sandewall10,DBLP:journals/jair/CasiniS13,DBLP:journals/jair/BonattiFS11}
they are silently neutralized.
\begin{example}[Double Diamond]
 Let $\KB$ be the following set of axioms:

\noindent
{\small
\begin{minipage}{0.5\textwidth}
\begin{eqnarray}
\label{axiom:B1first} \mathtt{A}&\sqsubseteq_n & \mathtt{T} 
\\
\mathtt{A} &\sqsubseteq_n & \mathtt{P}	
\\
\label{ex-DD-1}
 \mathtt{T} &\sqsubseteq_n & \mathtt{S}
\\
\label{ex-DD-2}
\mathtt{P}	 &\sqsubseteq_n & \neg\mathtt{S}	 
\end{eqnarray}
\end{minipage}
\begin{minipage}{0.5\textwidth}
\begin{eqnarray}
\label{ex-DD-3}
 \mathtt{S} &\sqsubseteq_n & \mathtt{R}\\
\mathtt{P} &\sqsubseteq_n & \mathtt{Q}\\
\label{axiom:B1last} \mathtt{Q} &\sqsubseteq_n  & \neg\mathtt{R}
\end{eqnarray}
\end{minipage}
}

\vspace{10pt}\noindent
DIs (\ref{ex-DD-1}) and (\ref{ex-DD-2}) have incomparable priority
under (\ref{specificity}) and (\ref{specificity-2}).
Consequently, it is easy to see that
$\N\mathtt{A}\sqsubseteq \mathtt{S}$ and
$\N\mathtt{A}\sqsubseteq \neg \mathtt{S}$ are both implied by $\KB^\Sigma$ and hence the knowledge engineer is warned that $\N\mathtt{A}$ is inconsistent.
The same conflict is silently neutralized in \cite{DBLP:journals/jair/CasiniS13,DBLP:journals/ai/Sandewall10,DBLP:journals/jair/BonattiFS11} 
($\mathtt{A}$'s instances are subsumed by neither  $\mathtt{S}$ nor  
$\neg \mathtt{S}$ and no inconsistency arises). Similarly for the incomparable DIs (\ref{ex-DD-3}) and (\ref{axiom:B1last}) and the related conflict.
\end{example}


The third category of examples (e.g.\  B.2 and B.3 in \cite{DBLP:journals/ai/Sandewall10}) presents a more variegated behavior. 
In particular, \cite{DBLP:journals/jair/CasiniS13} and \DLN with priority (\ref{specificity-2}) solve all conflicts and infer the same consequences;  \cite{DBLP:journals/ai/Sandewall10} solves only some conflicts; \cite{DBLP:journals/jair/BonattiFS11}  
 is not able to solve any conflict and yet it does not raise any inconsistency warning;  \DLN with priority (\ref{specificity}) cannot solve the conflicts but  raises an inconsistency warning.
Here, for the sake of simplicity,  we discuss in detail a shorter example which has all relevant ingredients. 

\begin{example}
\label{example:allDIs}
Let $\KB$ be the following defeasible knowledge base:

\noindent
{\small
\begin{minipage}{0.3\textwidth}
\begin{eqnarray}
\label{axiom:first} \mathtt{A}&\sqsubseteq_n & \mathtt{B}
\end{eqnarray}
\end{minipage}
\begin{minipage}{0.3\textwidth}
\begin{eqnarray}
\label{axiom:middle} \mathtt{A}&\sqsubseteq_n & \mathtt{C}
\end{eqnarray}
\end{minipage}
\begin{minipage}{0.3\textwidth}
\begin{eqnarray}
\label{axiom:last} \mathtt{B} &\sqsubseteq_n & \neg\mathtt{C}
 \end{eqnarray}
\end{minipage}
}

\vspace{10pt}\noindent
According to priority (\ref{specificity}) all DIs are incomparable. Therefore, 
\DLN warns 
(by inferring $\N \mathtt{A} \sqsubseteq \bot$) that  the conflict  between 
$\N\mathtt{A}\sqsubseteq \mathtt{C}$ and $\N\mathtt{A}\sqsubseteq \neg\mathtt{C}$ cannot be solved. Note that \cite{DBLP:journals/jair/BonattiFS11} adopts priority (\ref{specificity}), too, however according to circumscription, any  interpretation where $\mathtt{A}$'s instances are either in  $\mathtt{\neg C}\sqcap\mathtt{B}$ or in $\mathtt{C}$ is a model, so $\mathtt{A}$ is satisfiable (the conflict is silently neutralized).
Under priority (\ref{specificity-2}), instead, axiom (\ref{axiom:first}) gives (\ref{axiom:first}) and (\ref{axiom:middle}) higher priority than (\ref{axiom:last}). Consequently,  $\N\mathtt{A}\sqsubseteq \mathtt{C}$ prevails over
$\N\mathtt{A}\sqsubseteq \neg\mathtt{C}$. In this case, \DLN and rational closure infer the same consequences.
\end{example}

%
%

  \section{Relevance and modularity}
  \label{sec:module-extractor}

The naive construction of $\KB^\Sigma$ must process all the axioms in
$\KBall^\Sigma = \KB^\Sigma_0 \cup \{\delta^{\N C} \mid \delta \in
\D,\ \N C \in \Sigma\}$.
Here we optimize \DLN inference by quickly discarding some of the
irrelevant axioms in $\KBall^\Sigma$ using modularization techniques.

Roughly speaking, the problem of module extraction can be expressed as follows: given
a reference vocabulary $\Sig$, a module is a (possibly minimal) subset 
$\M\subseteq \KB$ that is relevant for $\Sig$ in the sense that it preserves
the consequences of $\KB$ that contain only terms in $\Sig$.


The interest in module extraction techniques is motivated by several
ontology engineering needs.
We are interested in modularization as an optimization technique for
querying large ontologies: the query is evaluated on a (hopefully much smaller) module of the ontology that preserves the query result (as well as any inference whose signature is contained in the query's signature).


However, the problem of deciding whether two knowledge bases 
entail the same axioms over a given signature is usually harder than standard reasoning tasks.  
Consequently deciding whether $\KB'$ is a module of $\KB$ (for $\Sig$) is computationally expensive in general.
For example, $DL\mbox{--}Lite_{horn}$ complexity grows from PTIME to coNP-TIME-complete \cite{DBLP:conf/kr/KontchakovWZ08}; 
for $\mathcal{ALC}$, complexity is one exponential harder \cite{GhilardiLutzWolter-KR06}, while for $\mathcal{ALCQIO}$ the problem becomes even undecidable \cite{DBLP:conf/ijcai/LutzWW07}.

In order to achieve a practical solution, a syntactic approximation has been adopted in \cite{DBLP:conf/dlog/SattlerSZ09,DBLP:journals/jair/GrauHKS08}. 
The corrisponding algorithm \tbs-$\Mod(\Sig,\KB)$ is defined
in \cite[Def.~4]{DBLP:conf/dlog/SattlerSZ09} and reported in
 Algorithm~\ref{alg:tbs-mod} below. 
It is based on the property of $\bot$-locality and $\top$-locality of single axioms (line \ref{alg:cond}).
An axiom is local w.r.t. $\Sig$ 
if the substitution of all non-$\Sig$ terms with $\bot$ (resp. $\top$)
turns it into a tautology. 

The module extractor identifies a subset $\M\subseteq\KB$ of the
knowledge base and a signature $\Sig$ (containing all symbols of
interest) such that all axioms in $\KB\setminus\M$ are local
w.r.t.\ \Sig. This guarantees that every model of \M can be extended
to a model of \KB by setting each non-$\Sig$ term to either $\bot$ or
$\top$. In turn, this property guarantees that any query whose
signature is contained in \Sig has the same answer in \M and \KB.

The function $x\mbox{-}\Mod(\Sig,\KB)$ (lines 9-19), where $x$ stands for $\top$ or $\bot$, 
describes the procedure for 
constructing modules of a knowledge base $\KB$ for each notion of locality. 
%
Starting with an empty set of axioms  (line 11), iteratively,  
the axioms $\alpha$  that are non-local are added to the module (line \ref{alg:than1}) and,
in order to preserve soundness, 
the signature against which locality is checked 
is extended with the terms in $\alpha$  (line \ref{alg:cond}).  Iteration stops when a fixpoint is reached.

Modules based on a single syntactic locality can be further shrinked
by iteratively nesting $\top$-extraction into $\bot$-extraction, thus obtaining \tbs-$\Mod(\Sig,\KB)$ modules (lines \ref{tb-start}-\ref{tb-end}).

The notions of module and
locality must be extended to handle DIs, before we can apply them to
\DLN. 
Definition \ref{def:mod-loc}
generalizes the substitutions operated by the module extraction
algorithm, abstracting away procedural details. 
As in \cite{DBLP:conf/dlog/SattlerSZ09}, both
$\widetilde X$ and $\sig(X)$ denote the signature of $X$.

\begin{algorithm}[t]
  \label{alg:tbs-mod}
  \caption{\tbs-$\Mod(\Sig,\KB)$}
  \small
  \dontprintsemicolon
  \KwIn{Ontology $\KB$,\quad signature \Sig}
  \KwOut{\tbs-module \M of \KB w.r.t.\ \Sig}
  \BlankLine
  \tcp{main}
  \Begin{ \label{tb-start}
         $\M:=\KB$\;
         \Repeat{$\M\ne \M'$}{ 
		 $\M':= \M$\;
           $\M := \mbox{$\top$-Mod}(\mbox{$\bot$-Mod}(\M,\Sig),\Sig)$\; 
        }
        \Return{\M}
    } \label{tb-end}
  \BlankLine
  \textbf{function} $x$-Mod(\KB,\Sig) 
  \tcp{\hspace{1em}$x\in\{\bot,\top\}$} \;
  \Begin{
      $\M := \emptyset$,\quad $\T := \KB$ \;
      \Repeat{\emph{changed = \texttt{false}}}{ 
        changed = \texttt{false} \;
        \ForAll{$\alpha\in\T$}{
          \If{ 
                 $\alpha$ is \emph{not} 
            x-local w.r.t. $\Sig \cup \widetilde{\M}$
\label{alg:cond}
} 
	   {
            $\M := \M \cup \{\alpha\}$\; \label{alg:than1}
           \label{alg:than2}$\T := \T \setminus \{\alpha\}$ \;
            changed = \texttt{true}
          }
        }
      }
      \Return{\M}
   }
\end{algorithm}
\footnotetext{\label{approx-locality}For efficiency, this test is
  approximated by a matching with a small set of templates.}

\begin{definition}\textbf{(Module, locality)}
  \label{def:mod-loc}
  A \emph{\tbs-substitution} for \KB and a signature
  \Sig is a substitution $\sigma$ over $\widetilde\KB \setminus
  \Sig$  that maps
  each concept name on $\top$ or $\bot$, and every role name on the
  universal role or the empty role.
  A strong axiom $\alpha$ is \emph{$\sigma$-local} iff
  $\sigma(\alpha)$ is a tautology. A DI $C\sqsubseteq_n D$ is
  $\sigma$-local iff $C\sqsubseteq D$ is $\sigma$-local. A set of
  axioms is $\sigma$-local if all of its members are.
We say that an axiom $\alpha$ is $\top$-local  (resp. $\perp$-local ) if $\alpha$ is $\sigma$-local where
  the substitution $\sigma$ uniformly maps concept names to $\top$
 (resp. $\perp$).

  A \emph{(syntactic) module} of  \KB with respect to
  \Sig is a set $\M \subseteq \KB$ such that $\KB \setminus \M$ is
  $\sigma$-local for some \tbs-substitution $\sigma$ for \KB and
  $\widetilde\M \cup \Sig$.
\end{definition}
	
Let $\Modd(\Sig,\KB)$ be the variant of \tbs-$\Mod(\Sig,\KB)$ where
the test in line \ref{alg:cond} is replaced with (the complement of) the
$\top$ or $\bot$-locality condition of Def.~\ref{def:mod-loc} (that covers
DIs, too). 
Using the original correctness argument for
\tbs-$\Mod(\Sig,\KB)$ cf. \cite[Prop. 42]{DBLP:journals/jair/GrauHKS08}, it is easy to see that $\Modd(\Sig,\KB)$ returns a
syntactic module of \KB w.r.t.\ \Sig according to
Def.~\ref{def:mod-loc}.  If \KB contains no DIs (i.e.\ it is classical), then
Def.~\ref{def:mod-loc} is essentially a rephrasing of standard
syntactic notions of modules and locality,\footnote{Informally, \tbs-Mod's greedy strategy tends to find small Def.~\ref{def:mod-loc}'s modules.} so
  \begin{equation}
   \label{mod-property} 
    \mbox{for all queries $\alpha$ such
      that $\widetilde\alpha \subseteq \Sig$, $\M \models \alpha$ iff $\KB
      \models \alpha$.}
  \end{equation}%

\noindent
However, proving that \tbs-$\Modd(\Sig,\KB)$ is correct for
\emph{full} \DLN is far from obvious: removing axioms from \KB using
module extractors is incorrect under most nonmonotonic semantics
(including circumscription, default logic and autoepistemic
logic). The reason is that nonmonotonic inferences are more powerful
than classical inferences, and the syntactic locality criterions
illustrated above fail to capture some of the dependencies between
different symbols.

\begin{example}
  Given the knowledge base $\{ \top \sqsubseteq A\sqcup B\}$ and
  $\Sig=\{A\}$, the module extractor returns an empty module (because
  by setting $B=\top$ the only axiom in the KB becomes a
  tautology). The circumscription of this KB, assuming that both $A$
  and $B$ are minimized, does not entail $A\sqsubseteq\bot$, while the
  circumscription of the empty module entails it.
\end{example}
\TRONLY{
\begin{example}
  Under the stable model semantics, the normal logic program
  \begin{eqnarray*}
    p & \leftarrow & \neg p, q \\
    q & \leftarrow & \neg r \\
    r & \leftarrow & \neg q
  \end{eqnarray*}
  entails $r$, both credulously and skeptically. The module extractor,
  given $\Sig=\{r\}$, removes the first rule (that becomes a tautology
  by setting $p=$true).  The module does not entail $r$ skeptically
  anymore, and erroneously entails $q$ credulously. Analogues of this
  example apply to default and autoepistemic logic, using the usual
  translations. They can be adapted to default DL
  \cite{DBLP:journals/jar/BaaderH95a} and the extension of DL based on
  MKNF \cite{DBLP:journals/tocl/DoniniNR02}.
\end{example}
}

Now we illustrate the correct way of applying \tbs-\Modd to a \DLN 
$\KB=\tmS\cup\D$ and a query $\alpha$ (subsumption or assertion). 
Let $\Sigma$ be the union of $\widetilde\alpha$ and the set of normality
concepts occurring in \KB.
Let 
{
  \[
  \M_0 = \Modd(\Sigma,\KB \cup \N\Sigma)\,,
  \]%
}%

\noindent
where $\N\Sigma$ abbreviates $\{\N C\sqsubseteq C \mid \N C \in \Sigma\}$.

\begin{example}
Let \KB be the knowledge base: 

\noindent
{\small
\begin{minipage}{0.5\textwidth}
\begin{eqnarray}
A &\sqsubseteq & B  \label{ex:a1}
\\
A &\sqsubseteq_n & D\sqcap E \label{ex:a2}
\end{eqnarray}
\end{minipage}
\begin{minipage}{0.5\textwidth}
\begin{eqnarray}
B\sqcap C &\sqsubseteq & A \label{ex:a3}
\\ 
F &\sqsubseteq_n & A \label{ex:a4}
\end{eqnarray}
\end{minipage}
}

\vspace{10pt}
\noindent
and $\alpha$ the query $\N A\sqsubseteq D$. 
$\M_0$ is calculated as follows: 
first, since no normality concept
occurs in $\KB$, $\Sigma$ is equal to the signature 
$\widetilde\alpha=\{\N A, D\}$.  

Algorithm  \ref{alg:tbs-mod} calls first the function
$\perp$-Mod($\KB\cup \N \Sigma$,$\Sigma$). 
Notice that by replacing $C$ and $F$ with $\perp$, axioms 
(\ref{ex:a3}) and (\ref{ex:a4}) becomes a tautology. Consequently, 
it is easy to see that the returned knowledge base is $\KB'=\{(\ref{ex:a1}), (\ref{ex:a2}), \N A \sqsubseteq  A\}$.

Then, $\top$-Mod is called on $\KB'$ and $\Sigma$. Now, 
replacing $B$ with $\top$ makes  $A \sqsubseteq  B$ a tautology,
so the resulting knowledge base is $\KB''=\{(\ref{ex:a2}), \N A \sqsubseteq  A\}$. It is easy to see that a fix point is reached and 
hence $\KB''$ is returned. 
\end{example}

We shall prove that $(\KB\cap\M_0)^\Sigma$ can be used in place of
$\KB^\Sigma$ to answer query $\alpha$. This saves the cost of
processing $\KBall^\Sigma \setminus \M$, where
{\small
  \[
  \M = (\KB^\Sigma_0 \cap \M_0) \cup \{\delta^{\N C} \mid \delta \in
  \D\cap\M_0,\ \N C\in\Sigma\}.
  \]%
}%

\noindent
Note that $\KBall^\Sigma \setminus \M$ is usually even larger
than $\KB\setminus\M_0$ because for each DI $\delta \not\in\M_0$,
all its translations $\delta^{\N C}$ ($\N C\in\Sigma$) are removed from \M.

\begin{lemma}
  \label{lem:M-module}
  \M is a module of $\KBall^\Sigma$ w.r.t.\ $\Sigma$.
\end{lemma}
\TRONLY{
\begin{proof}
  Since $\Modd(\cdot,\cdot)$ returns modules, \hide{by
  Def.~\ref{def:mod-loc} the set} $(\KB \cup \N\Sigma) \setminus \M_0$
  is $\sigma$-local, for some \tbs-substitution $\sigma$ for $\KB \cup
  \N\Sigma$ and $\widetilde\M_0\cup\Sigma$. So, for all axioms $\beta$
  in $\KB^\Sigma_0 \setminus \M$, $\beta$ is $\sigma$-local (as
  $\KB^\Sigma_0 \setminus \M \subseteq (\KB \cup \N\Sigma) \setminus
  \M_0$).  Moreover, for all $\beta = \delta^{\N D} \in \{\delta^{\N
    C} \mid \delta \in \D,\ \N C \in \Sigma\} \setminus \M$, it must
  hold $\delta\in\D\setminus \M_0$ (by construction of \M), and hence
  $\delta$ is $\sigma$-local.  Now note that if $\sigma(E\sqsubseteq
  F)$ is a tautology, then also $\sigma(\N D \sqcap E\sqsubseteq F)$
  is a tautology, therefore the $\sigma$-locality of $\delta$ implies
  the $\sigma$-locality of $\delta^{\N D}$. It follows that all
  $\beta$ in $\KBall^\Sigma \setminus \M$ are $\sigma$-local.

  Finally, note that $\sig(\KBall^\Sigma) = \sig(\KB \cup \N\Sigma)$
  and $\widetilde\M \cup \Sigma = \widetilde\M_0 \cup \Sigma$, therefore
  $\sigma$ is also a \tbs-substitution for $\KBall^\Sigma$ and $\M
  \cup \Sigma$. It follows immediately that \M is a module of
  $\KBall^\Sigma$ w.r.t.\ $\Sigma$.
\qed
\end{proof}
}

\begin{lemma}
  \label{lem:sub-modules}
  If \M is a module of \KB w.r.t.\ a signature \Sig and $\KB' \subseteq \KB$, then $\KB' \cap \M$ is a module of $\KB'$  w.r.t.\ \Sig.
\end{lemma}
\TRONLY{
\begin{proof}
  If \M is a module of \KB w.r.t.\ \Sig, then $\KB\setminus\M$ is
  $\sigma$-local for some \tbs-substitution $\sigma$ for \KB and
  $\widetilde\M \cup\Sig$.  Let $\sigma'$ be the restriction of $\sigma$
  to the symbols in 
  \[
  \widetilde\KB'\setminus(\widetilde\M\cup\Sig) = \widetilde\KB'\setminus(\sig(\KB'\cap\M)\cup\Sig).
  \]
  Clearly, $\sigma'$ is a \tbs-substitution for $\KB'$ and
  $\widetilde\KB'\setminus(\sig(\KB'\cap\M)\cup\Sig)$. Moreover, for all
  $\beta \in \KB'\setminus\M$, $\sigma(\beta) = \sigma'(\beta)$, by
  construction, so $\KB'\setminus\M$ is $\sigma'$-local. Then $\KB' \cap \M$ is a syntactic module of $\KB'$ w.r.t.\ \Sig.
\qed
\end{proof}
}

\noindent
The relationship between $(\KB\cap\M_0)^\Sigma$ and $\KB^\Sigma$ is:

\begin{lemma}
  \label{lem:mod-correctness}
  $\KB^\Sigma \cap \M \subseteq (\KB \cap \M_0)^\Sigma \subseteq \KB^\Sigma \,.$
\end{lemma}
\TRONLY{
\begin{proof}
  It suffices to prove by induction that
  for all $i=0,1,\ldots,|\D|$,\footnote{\label{modified-constr}Here the
    sets $(\KB\cap\M_0)^\Sigma_i$ are defined by replacing \tmS with
    $\tmS\cap\M_0$, and \KB with $\KB\cap\M_0$ in
    (\ref{KB-Sigma-constr-1}) and (\ref{KB-Sigma-constr-2}), while
    $\delta_i$ ranges over \emph{all} the DIs in \KB, not just
    $\D\cap\M_0$. This formulation facilitates the comparison with
    $\KB^\Sigma$. By the condition $\delta_i\in\KB$ in (\ref{KB-Sigma-constr-2})
    for all $\delta_i\not\in\M_0$, $(\KB\cap\M_0)^\Sigma_i =
    (\KB\cap\M_0)^\Sigma_{i-1}$, so this definition is equivalent to
    building $(\KB\cap\M_0)^\Sigma$ using only $\D \cap\M_0$.}
  \[
  \KB^\Sigma_i \cap \M \subseteq (\KB \cap \M_0)^\Sigma_i \subseteq \KB^\Sigma_i \,.
  \]

  \noindent
  \emph{Base case} ($i=0$): 
  {\small
  \begin{eqnarray*}
    \underline{\KB^\Sigma_0 \cap \M} & = & (\tmS \cup \N\Sigma) \cap \M  ~~\subseteq ~~ 
    (\tmS \cap \M)  \cup \N\Sigma ~~=
    \\
    &=&  (\tmS \cap \M_0)  \cup \N\Sigma ~~=~~ 
    \underline{(\KB \cap \M_0)^\Sigma_0} ~~\subseteq   
    \tmS \cup \N\Sigma ~~=~~ \underline{\KB^\Sigma_0} \,.
  \end{eqnarray*}%
  }%

  \noindent
  \emph{Induction step} ($i>0$): By induction hypothesis (IH)
  \[
  \KB^\Sigma_{i-1} \cap \M \subseteq (\KB \cap \M_0)^\Sigma_{i-1} 
  \subseteq \KB^\Sigma_{i-1} \,.
  \]
  First suppose that $\delta_i\not\in\M_0$ (and hence for all $\N C$, $\delta_i^{\N C}\not\in\M$). Then
  $\KB^\Sigma_i \cap \M = \KB^\Sigma_{i-1} \cap \M$, $(\KB \cap
  \M_0)^\Sigma_i = (\KB \cap \M_0)^\Sigma_{i-1}$, and (by def.)
  $\KB^\Sigma_{i-1} \subseteq \KB^\Sigma_i$. The Lemma follows by
  IH.

  Next assume that $\delta_i\in\M_0$ and let $\N C$ be any normality
  concept in $\Sigma$. Note that $\delta_i^{\N C}\in\M$. By IH, 
  {\small
    \[
    (\KB^\Sigma_{i-1} \cap \M) \downarrow_{\prec\delta_i} 
    \cup \{\delta_i^{\N C}\} \subseteq
    (\KB \cap \M_0)^\Sigma_{i-1} \downarrow_{\prec\delta_i} 
    \cup \{\delta_i^{\N C}\} \subseteq \KB^\Sigma _{i-1} \downarrow_{\prec\delta_i} 
    \cup \{\delta_i^{\N C}\} \,.
    \]%
  }%

\noindent
The leftmost term equals $(\KB^\Sigma _{i-1}
\downarrow_{\prec\delta_i} \cup \{\delta_i^{\N C}\}) \cap \M \subseteq \KBall^\Sigma$, so by
Lemmas \ref{lem:M-module} and \ref{lem:sub-modules} and (\ref{mod-property}), the leftmost term entails $\N C\sqsubseteq
\bot$ iff the rightmost does. It follows that the middle term $(\KB \cap
\M_0)^\Sigma_{i-1} \downarrow_{\prec\delta_i} \cup \{\delta_i^{\N C}\}$ entails $\N C\sqsubseteq \bot$ iff the other two terms do. Then, \hide{by
(\ref{KB-Sigma-constr-2}),} $\delta_i^{\N C}$ is added to $(\KB \cap
\M_0)^\Sigma_i$ iff $\delta_i^{\N C}$ belongs to $\KB^\Sigma_i \cap
\M$ and $\KB^\Sigma_i$, and the Lemma follows using the IH.
\qed
\end{proof}
}

\noindent
As a consequence, the modularized construction is correct:
\begin{theorem}
  \label{thm:mod-correctness}
  $(\KB\cap\M_0)^\Sigma \models \alpha$ iff $\KB^\Sigma \models \alpha$.
\end{theorem}

\begin{proof}
  By Lemmas~\ref{lem:M-module} and \ref{lem:sub-modules}, and
  (\ref{mod-property}), $\KB^\Sigma \models \alpha$ iff $\KB^\Sigma
  \cap \M \models \alpha$. The Theorem then follows by
  Lemma~\ref{lem:mod-correctness}.
\qed
\end{proof}

  \section{Optimistic computation}
  \label{sec:optimistic-method}

The construction of $\KB^\Sigma$ repeats the concept consistency check
(\ref{KB-Sigma-constr-2}) over a sequence of knowledge bases
($\KB^\Sigma_{i-1}\downarrow_{\prec \delta_i} \cup\ \{\delta_i^{\N
  C}\}$) that share a (possibly large) common part $\KB_0^\Sigma$,
so incremental reasoning mechanisms help by avoiding multiple
computations of the consequences of $\KB_0^\Sigma$. On the contrary,
the set of $\delta_j^{\N C}$ may change significantly at each step due
to the filtering $\downarrow_{\prec \delta_i}$. This operation
requires many axiom deletions, which as already highlighted in \cite{DBLP:conf/dlog/KazakovK13}, 
are less efficient than monotonically increasing changes. The optimistic algorithm introduced
here (Algorithm~\ref{opt-alg}) computes a knowledge base $\KB^*$
equivalent to $\KB^\Sigma$ in a way that tends to reduce the number of deletions, as it will be assessed in Sec. \ref{sec:experiments}.

Phase~1 optimistically assumes that the DIs with the same priority as
$\delta_i^{\N C}$ do not contribute to entailing $\N C\sqsubseteq
\bot$ in (\ref{KB-Sigma-constr-2}), so they are not filtered with
$\downarrow_{\delta_i}$ in line \ref{t1}. Phase~2 checks whether the
DIs discarded during Phase~1 are actually overridden by applying
$\downarrow_{\delta_i}$ (lines \ref{t2} and \ref{t3}). DIs are
processed in non-increasing priority order as much as possible
(cf.\ line \ref{w1})
so as to exploit monotonic incremental classifications.

The following theorem shows the correctness of Alghorithm \ref{opt-alg} in case the normality concepts do not occur in \KB, but only in the queries. We call such knowledge bases \emph{N-free}. It is worth noting that the optimistic method is not generally correct when \KB is not N-free and $|\Sigma|>1$, yet it may still be applicable after the module extractor if the latter removes all normality concepts from \KB.

\begin{theorem}
  \label{thm:opt-correctness}
  If \KB is N-free, then Algorithm~\ref{opt-alg}'s output is equivalent to $\KB^\Sigma$.
\end{theorem}
\TRONLY{
\begin{proof}
  First assume that $\Sigma$ is a singleton $\{\N C\}$. We start by
  proving some invariants of lines 6-10.

\textbf{Claim 1}: $\KB^\Sigma \models \Pi$.

\textbf{Claim 2}: If, for some $j<i$, $\delta_j^{\N C} \in \KB^\Sigma
\setminus \Pi$, then $\KB^\Sigma \models \N C \sqsubseteq \bot$.

We prove these two claims simultaneously.  Both claims hold vacuously
at the first execution of line 6. Next, assume by induction hypothesis
that they hold at line 6 in some iteration; we have to prove that they
still hold at the next iteration. There are two possibilities: First suppose
that for some $j<i$, $\delta_j^{\N C} \in \KB^\Sigma \setminus
\Pi$. By Claim~2, $\KB^\Sigma \models \N C \sqsubseteq \bot$. This
immediately implies that Claim~2 holds also at the next
iteration. Moreover, it implies Claim~1 because all members of $\Pi$
have an occurrence of $\N C$ in the left-hand side,
cf.\ (\ref{DI-trans}).

We are left with the case in which 
\begin{equation}
  \label{opt-corr-1}
  \mbox{ for all $j<i$, if $\delta_j^{\N C} \in \KB^\Sigma$
  then $\delta_j^{\N C} \in \Pi$. }
\end{equation}
If the condition in line \ref{t1} is true, then $\Pi$ is not changed, so
Claim~1 must hold at the next iteration. Otherwise, by
(\ref{opt-corr-1}), 
\begin{equation}
  \label{opt-corr-2}
  \KB^\Sigma_0 \cup \Pi' \supseteq
  \KB^\Sigma_{i-1}\downarrow_{\delta_i} \cup \{\delta_i^{\N C}\}\,, 
\end{equation}
and hence $\N C \sqsubseteq \bot$ is not provable in
(\ref{KB-Sigma-constr-2}), either.  It follows that $\delta_i^{\N C}$
belongs to both $\Pi$ (by line \ref{add1}) and $\KB^\Sigma$ (by
(\ref{KB-Sigma-constr-2})). This proves Claim~1 for iteration
$i$. 

Concerning Claim~2, first suppose that the condition in line \ref{t1}
is true; then either Claim~2 remains vacuously satisfied, or
$\delta_i^{\N C} \in \KB^\Sigma \setminus \Pi$. The latter (plus the
def.\ of $\KB^\Sigma$ and the ind.\ hyp.\ of Claim~1) implies that
$\KB^\Sigma_0 \cup \Pi'$ is entailed by $\KB^\Sigma$. It follows that
$\KB^\Sigma \models \N C \sqsubseteq \bot$ as well, which proves
Claim~2 in this case.  Finally, if the condition in line \ref{t1} is
false, then at line~\ref{add1} $\delta_i^{\N C} \in \Pi$. Together
with (\ref{opt-corr-1}), this implies that Claim~2 holds vacuously.

\textbf{Claim 3}: If $\KB^\Sigma \models \N C\sqsubseteq \bot$ then $\KB^* \models \N C \sqsubseteq \bot$.

Suppose not (we shall derive a contradiction). The assumption and
Claim~1 imply that there must be some $\delta_k^{\N C} \in \KB^\Sigma
\setminus \KB^*$. Let $\delta_i^{\N C}$ be the one with minimal
$k$. Using minimality, it can be proved that $\KB_0^\Sigma \cup \Pi
\downarrow_{\delta_i} = \KB^\Sigma_{i-1} \downarrow_{\delta_i}$, so the
concept consistency tests in lines \ref{t2} and \ref{t3} (the latter
instantiated with $j=i$ and $D=C$) are equivalent to the one in
(\ref{KB-Sigma-constr-2}). But then $\delta_i^{\N C}\in \KB^\Sigma$
iff $\delta_i^{\N C}\in \KB^*$, which contradicts the assumption, so
Claim~3 is proved.

\textbf{Claim 4}: If $\KB^\Sigma \models \N C\sqsubseteq \bot$ then
$\KB^* \equiv \KB^\Sigma$.

Note that $\KB^* \subseteq \KB^\Sigma_0 \cup \Pi \cup \{\N C
\sqsubseteq \bot\}$ (cf.\ lines 11, 15, and 22). Clearly, $\KB^\Sigma
\models \KB^\Sigma_0 \cup \Pi \cup \{\N C \sqsubseteq \bot\}$ (by
def., Claim~1 and the assumption), so $\KB^\Sigma \models \KB^*$.  We
are left to prove $\KB^* \models \KB^\Sigma$. By Claim~3, $\KB^*
\models \N C \sqsubseteq \bot$, and this inclusion in turn entails
all $\delta_i^{\N C}\in \KB^\Sigma$
(cf. (\ref{KB-Sigma-constr-1})). The other members of $\KB^\Sigma$ are
those in $\KB^\Sigma_0$, by definition, and $\KB^* \supseteq
\KB^\Sigma_0$ (line 11). It follows that $\KB^* \models \KB^\Sigma$,
which completes the proof of Claim~4.

\textbf{Claim 5}: If $\KB^\Sigma \not\models \N C\sqsubseteq \bot$ then
$\KB^* \equiv \KB^\Sigma$.

Suppose that $\KB^\Sigma \not\models \N C\sqsubseteq \bot$.  Then, by
the contrapositive of Claim~2 and Claim~1, $\KB^\Sigma_0\cup \Pi
\equiv \KB^\Sigma$, which further implies that the concept consistency
tests in lines \ref{t2} and \ref{t3} are equivalent to the
corresponding test in (\ref{KB-Sigma-constr-2}). Then it can be proved
that if any of these tests were true, then also $\KB^\Sigma \models \N
C \sqsubseteq \bot$ because, by $\Pi$'s construction, in that case $\delta_i$ must be in conflict with some other
DI with the same priority. However, $\KB^\Sigma \models \N C
\sqsubseteq \bot$ contradicts the assumption. It follows that all
tests in lines \ref{t2} and \ref{t3} are false, so
$\KB^*=\KB^\Sigma_0\cup \Pi$, and we have already argued that this
knowledge base is equivalent to $\KB^\Sigma$.
This completes the proof for
$|\Sigma|=1$.

For $|\Sigma|>1$, note that the tests in lines \ref{t1} and \ref{t2}
(resp. \ref{t3}) do not depend on any $\delta_k^{\N E}$ such that
$E\neq C$ (resp.\ $E\neq D$). Indeed, by \emph{$\bot$-locality}, all
such $\delta_k^{\N E}$ are local w.r.t.\ the signature of $\KB \cup
\{\N C \sqsubseteq \bot\}$, so they can be removed without changing
the result of the concept consistency test
\cite{DBLP:conf/dlog/SattlerSZ09}. However, after their removal, the
concept consistency tests correpond to the ones for the singleton case
$\Sigma=\{\N C\}$, which we have already proved correct.
\qed
\end{proof}
}

\begin{algorithm}
  \caption{Optimistic-Method \label{opt-alg}}
  \small
  \dontprintsemicolon
  \KwIn{$\KB=\tmS\cup\D$, $\Sigma$}
  \KwOut{a knowledge base $\KB^*$ such that $\KB^*\equiv\KB^\Sigma$}
  \BlankLine
  \tcp{Phase 1}
  compute a linearization $\delta_1,\ldots,\delta_{|\D|}$ of \D\;
  $\Pi := \emptyset$ \tcp{$\Pi$ collects the prototypes}
  $\Delta := \emptyset$ \tcp{ordered list of all discarded $\delta_i^{\N C}$}
  \For{$i=1,2,\ldots,|\D|$}{
    \For{$\N C \in \Sigma$}{
      $\Pi' := \Pi \cup \{\delta_i^{\N C}\}$\;
      \eIf{$\KB_0^\Sigma \cup \Pi' \not\models \N C\sqsubseteq \bot$\label{t1}}{
        $\Pi := \Pi'$ \label{add1}}{
        append $\delta_i^{\N C}$ to $\Delta$}
    }
  }
  \tcp{Phase 2}
  $\KB^* = \KB_0^\Sigma \cup \Pi$\;
  \While{$\Delta \neq \emptyset$}{
    extract from $\Delta$ its first element $\delta_i^{\N C}$\;
    \eIf{$(\KB_0^\Sigma \cup \Pi)\downarrow_{\prec \delta_i} 
      \cup \{\delta_i^{\N C}\} \not\models \N C\sqsubseteq \bot$ \label{t2}}{
      $\KB^* := \KB^* \cup \{\N C\sqsubseteq \bot\}$\;
      extract all $\delta_k^{\N E}$ with $E=C$ from $\Delta$
    }{
      \tcp{$\delta_i^{\N C}$ is actually overridden}
      $\delta := \delta_i$\;
      \While{$\Delta$ contains some $\delta_j^{\N D}$ such that 
        $\delta\prec\delta_j$ \label{w1}}{
        extract from $\Delta$ the first such $\delta_j^{\N D}$\;
        \If{$(\KB_0^\Sigma \cup \Pi)\downarrow_{\prec \delta_j} 
          \cup \{\delta_j^{\N D}\} \not\models \N D\sqsubseteq \bot$ \label{t3}}{
          $\KB^* := \KB^* \cup \{\N D\sqsubseteq \bot\}$\;
          extract all $\delta_k^{\N E}$ with $E=D$ from $\Delta$\;
          $\delta := \delta_j$\; \label{w2}
        }
      }
    }
  }
\end{algorithm}

  \section{Experimental assessment}
  \label{sec:experiments}

Currently there are no ``real'' KBs encoded in a nonmonotonic DL,
because standard DL technology does not support nonmonotonic
reasoning. The nonmonotonic KBs encoded in the hybrid rule+DL system
DLV-Hex \cite{DBLP:series/lncs/DrabentEIKLM09} are not suited to our
purposes because they do not feature default inheritance due to a
restriction of the language: DL predicates cannot occur in rule heads,
so rules cannot be used for encoding default inheritance. Consequently,
synthetic test cases are the only choice for evaluating our
algorithms. We start with the two test suites introduced in
\cite{DLN-15} as they have been proved to be nontrivial w.r.t.\ a
number of structural parameters, including nonclassical features like
exception levels and the amount of overriding. The two test suites are
obtained by modifying the popular Gene Ontology
(GO)\footnote{http://www.geneontology.org}, which contains 20465
atomic concepts and 28896 concept inclusions. In one test suite,
randomly selected axioms of GO are turned into DIs, while in the
second suite random synthetic DIs are injected in GO. The amount of
strong axioms transformed into DIs is controlled by
\emph{CI-to-DI-rate}, expressed as the percentage of transformed
axioms w.r.t.\ $|$GO$|$ while the amount of additional synthetic DIs
is controlled by \emph{Synthetic-DI-rate}, i.e.\ the ratio
$|\D|/|$GO$|$. The number of conflicts between DIs can be increased by
adding an amount of random disjointness axioms specified by parameter
\emph{DA-rate} (see \cite{DLN-15} for further details).

The experiments were performed on an Intel Core i7 2,5GHz laptop with
16 GB RAM and OS X 10.10.1, using Java 1.7 configured with 8 GB RAM and 3
GB stack space. Each reported value is the average execution time over
ten nonmonotonic ontologies and fifty queries on each ontology.
For each parameter setting, we report the execution time of: ($i$) the
naive \DLN reasoner of \cite{DLN-15}; ($ii$) the optimistic method
introduced in Sec.~\ref{sec:optimistic-method} (\opt); ($iii$) the
module extraction method of Sec.~\ref{sec:module-extractor} (\modex)
using the module extraction facility of the OWLAPI; ($iv$) the
sequential execution of \modex and \opt, i.e.\ Algorithm~\ref{opt-alg}
is applied to $\KB\cap\M_0$. This combined method is correct by
Theorem~\ref{thm:opt-correctness} and
Theorem~\ref{thm:mod-correctness}.

\begin{table}
  \begin{center}
    \footnotesize
    \begin{tabular}{|c||c|c|c|c|}
      \hline
      CI-to-DI & naive & opt & mod & mod+opt\\
      \hline
      \hline
      05\% & 12.91 & 05.93 & 00.30 & 00.25 \\
      10\% & 22.37 & 11.13 & 00.32 & 00.27 \\
      15\% & 31.50 & 15.90 & 00.37 & 00.32 \\
      20\% & 42.97 & 20.67 & 00.40 & 00.33 \\
      25\% & 55.22 & 25.17 & 00.44 & 00.36 \\
      \hline
    \end{tabular}
    \quad
    \begin{tabular}{|c||c|c|c|c|}
      \hline
      Synth DIs & naive & opt & mod & mod+opt\\
      \hline
      \hline
      05\% & 11.64 & 06.94 & 0.41 & 0.42 \\
      10\% & 21.66 & 11.21 & 0.62 & 0.67 \\
      15\% & 32.80 & 14.90 & 1.11 & 1.64 \\
      20\% & 41.51 & 18.82 & 2.01 & 1.42 \\
      25\% & 51.85 & 22.33 & 3.05 & 2.09 \\
      \hline
    \end{tabular}
  \end{center}
  \caption{Impact of $|\D|$ on performance (sec) -- DA rate = 15\% -- priority (\ref{specificity})}
  \label{tab:DI-rate}
\end{table}

Table~\ref{tab:DI-rate} shows the impact of the number of DIs
on response time for the two test suites, as DI rate ranges from 5\% to 25\%.
The methods \modex and \modex{}{}+\opt are slightly less effective in the second
suite probably because random
defaults connect unrelated parts of the ontology,
thereby hindering module extraction.
In both suites, \opt's speedup factor (w.r.t.\ the naive method) is
about two, while on average \modex is approximately 87 times faster in
the first test suite (max. speedup 125), and 28 times faster in the
second (max.\ speedup 35).  On average, the combined method yields a
further 13\% improvement over \modex alone; the maximum reduction is
31\% ($2^\mathrm{nd}$ suite, Synthetic-DI-rate=25\%, DA-rate=15\%).
The additional conflicts induced by injected disjointness axioms have moderate effects on response time (cf. Table~\ref{tab:DA-rate}).
\modex{}+\opt's average response time across both test suites is 0.7 sec.,
and the longest \modex{}+\opt response time has been 2.09 sec.  As a term of
comparison, a single classification of the original GO takes
approximately 0.4 seconds.

\ARTONLY{\vspace*{-1em}}
\begin{table}
  \begin{center}
    \footnotesize
    \begin{tabular}{|c||c|c|c|c|}
      \multicolumn{5}{c}{Test suite 1 (CI-to-DI)}\\
      \hline
      DA & naive & opt & mod & mod+opt\\
      \hline
      \hline
      05\% & 29.88 & 13.21 & 0.36 & 0.31 \\
      10\% & 32.96 & 14.08 & 0.37 & 0.32 \\
      15\% & 31.50 & 15.90 & 0.37 & 0.32 \\
      20\% & 34.23 & 16.23 & 0.39 & 0.33 \\
      25\% & 36.47 & 17.80 & 0.40 & 0.34 \\
      30\% & 37.71 & 18.09 & 0.40 & 0.34 \\
      \hline
    \end{tabular}
    \quad
    \begin{tabular}{|c||c|c|c|c|}
      \multicolumn{5}{c}{Test suite 2 (Synth.\ DIs)}\\
      \hline
      DA & naive & opt & mod & mod+opt\\
      \hline
      \hline
      05\% & 28.20 & 12.63 & 0.99 & 0.84 \\
      10\% & 30.18 & 13.68 & 1.04 & 0.97 \\
      15\% & 32.80 & 14.90 & 1.11 & 1.06 \\
      20\% & 35.68 & 16.29 & 1.18 & 1.10 \\
      25\% & 37.46 & 17.02 & 1.25 & 1.15 \\
      30\% & 38.37 & 18.79 & 1.36 & 1.23 \\
      \hline
    \end{tabular}
  \end{center}
  \caption{Impact of DAs on performance (sec) -- DI rate = 15\% -- priority (\ref{specificity})}
  \label{tab:DA-rate}
\end{table}

\ARTONLY{\vspace*{-4em}}

\begin{table}
  \begin{center}
    \footnotesize
    \begin{tabular}{|c||c|c|c|c|}
      \hline
      CI-to-DI & naive & opt & mod & mod+opt\\
      \hline
      \hline
      05\% & ~~22.01 & 05.74 & 00.30 & 00.25 \\
      10\% & ~~52.82 & 11.48 & 00.32 & 00.28 \\
      15\% & ~~81.84 & 16.56 & 00.34 & 00.31 \\
      20\% & 133.62 &  20.51 & 00.38 & 00.33 \\
      25\% & 193.27 &  26.42 & 00.41 & 00.36 \\
      \hline
    \end{tabular}
    \quad
    \begin{tabular}{|c||c|c|c|c|}
      \hline
      Synth DIs & naive & opt & mod & mod+opt\\
      \hline
      \hline
      05\% & 12.76 & 07.21 & 0.45 & 0.46 \\
      10\% & 23.72 & 14.44 & 0.81 & 0.86 \\
      15\% & 34.53 & 17.05 & 1.57 & 1.21 \\
      20\% & 44.92 & 21.77 & 2.67 & 1.96 \\
      25\% & 55.92 & 25.77 & 3.87 & 2.46 \\
      \hline
    \end{tabular}
  \end{center}
  \caption{Impact of $|\D|$ on performance (sec) -- DA rate = 15\% --
  priority (\ref{specificity-2})}
  \label{tab:DI-rate-2}
\end{table}
\vspace*{-2em}

Table~\ref{tab:DI-rate-2} is the analogue of Table~\ref{tab:DI-rate}
given priority (\ref{specificity-2}). With respect to priority
(\ref{specificity}), the computation time for $\KB^\Sigma$ and query
answering in the first test suite grows faster for the naive algorithm, while there are
smaller differences for the optimized approaches (the reponse times of the combined approach are almost identical). In the second test suite, the performance of the naive algorithms decreases less dramatically, while the optimized methods seem slightly less effective than in the first test suite. In all cases, the speedups of \modex and \modex-\opt  remain well above one order of magnitude.  The performance as DAs grow has similar features 
(see Table~\ref{tab:DA-rate-2}).

\begin{table}
  \begin{center}
    \footnotesize
    \begin{tabular}{|c||c|c|c|c|}
      \multicolumn{5}{c}{Test suite 1 (CI-to-DI)}\\
      \hline
      DA & naive & opt & mod & mod+opt\\
      \hline
      \hline
      05\% & 84.53 & 15.02 & 0.34 & 0.29 \\
      10\% & 90.38 & 16.12 & 0.35 & 0.30 \\
      15\% & 91.84 & 16.56 & 0.35 & 0.31 \\
      20\% & 92.93 & 16.67 & 0.36 & 0.31 \\
      25\% & 93.54 & 17.76 & 0.37 & 0.32 \\
      30\% & 96.37 & 19.49 & 0.38 & 0.33 \\
      \hline
    \end{tabular}
    \quad
    \begin{tabular}{|c||c|c|c|c|}
      \multicolumn{5}{c}{Test suite 2 (Synth.\ DIs)}\\
      \hline
      DA & naive & opt & mod & mod+opt\\
      \hline
      \hline
      05\% & 29.55 & 14.93 & 1.28 & 1.07 \\
      10\% & 30.81 & 15.82 & 1.41 & 1.15 \\
      15\% & 34.54 & 17.05 & 1.57 & 1.21 \\
      20\% & 36.79 & 16.93 & 1.62 & 1.27 \\
      25\% & 40.86 & 17.90 & 1.78 & 1.36 \\
      30\% & 43.35 & 18.74 & 1.79 & 1.34 \\
      \hline
    \end{tabular}
  \end{center}
  \caption{Impact of DAs on performance (sec) -- DI rate = 15\% --
  priority (\ref{specificity-2})}
  \label{tab:DA-rate-2}
\end{table}


The above test sets are N-free. We carried out a new set of
experiments by randomly introducing normality concepts
in DIs, within the scope of quantifiers.\footnote{\label{NC-role}So far, all the applicative examples that are not
  N-free satisfy this restriction, as apparently the only purpose of explicit normality concepts is restricting default role ranges to normal individuals,
  cf.\ Ex.~12 and the nomonotonic design pattern in
  \cite[Sec.~3.3]{DLN-15}.} Specifically, $\exists R.C$ is transformed into $\exists R.\N\,C$. 
The response times of the naive algorithm and \modex\footnote{In this
  setting \opt and \modex{}+\opt are not applicable, in general.} under priority (\ref{specificity}) are
listed in Table~\ref{tab:non-N-free} for increasing values of $|\Sigma|$ (that is directly related to the amount of normality concepts occurring in \KB).
We estimate that the values of $|\Sigma|$ considered here are
larger than what should be
expected in practice, given the specific role of explicit normality
concepts, cf.\ footnote~\ref{NC-role}. Such values are also much larger than in N-free experiments, where $|\Sigma|$ is
bounded by the query size.  Response times increase accordingly.
In most cases, the naive algorithm exceeded the
timeout. In the first test suite, \modex remains well below 1 minute; in the second suite it ranges between 100 seconds and 10 minutes. The reason of the higher computation times in the second suite is that the extracted modules are significantly larger, probably due to the random dependencies between concept names introduced by fully synthetic DIs.

\begin{table}
  \begin{center}
    \footnotesize
    \begin{tabular}{|c||c|c|c|c|c|}
      \hline
      $|\Sigma|$ & 50 & 100 & 150 & 200 & 250 \\
      \hline
      \multicolumn{6}{c}{Test suite 1}
      \\
      \hline
      naive & 1794.37 & $>$30 min. & $>$30 min. & $>$30 min. & $>$30 min. \\
      mod & 2.31 & 7.26 & 14.77 & 25.32 & 39.22 \\
      \hline
      \multicolumn{6}{c}{Test suite 2}
      \\
      \hline
      naive & $>$30 min. & $>$30 min. & $>$30 min. & $>$30 min. & $>$30 min. \\
      mod & 103.4 & 211.5 & 327.4 &  459.2 & 586.7  \\
      \hline
    \end{tabular}
  \end{center}
  \caption{Impact of normal roles values (sec) -- DI rate = 25\% DA rate = 15\%}
  \label{tab:non-N-free}
\end{table}

\vspace*{-3em}

  \section{Conclusions}
  \label{sec:conclusions}

The module-based and optimistic optimizations introduced here are
sound and complete, where the later applies only if the knowledge base is N-free. 
In our experiments, the combined method (when
applicable) and the module-based method make \DLN reasoning at least
one order of magnitude faster (and up to $\sim$780 times faster in
some case).
In most cases, optimized reasoning is compatible with real time \DLN
reasoning. This is the first time such performance is reached over
nonmonotonic KBs of this size: more than 20K concept names and over
30K inclusions.\footnote{Good results have been obtained also for KBs
  with $\sim$5200 inclusions under rational closure semantics
  \cite{DBLP:conf/dlog/CasiniMMV13,DBLP:conf/jelia/CasiniMMN14}.} Our
approach brings technology closer to practical nonmonotonic reasoning
with very large KBs. Only the random dependencies introduced by synthetic DIs, combined with numerous restrictions of role ranges to normal individuals, can raise response time over 40 seconds; in most of the other cases, computation time remains below 2 seconds.

We are currently exploring a more aggressive module extraction approach,
capable of eliminating some of the normality concepts in $\Sigma$ and
related axioms. Besides improving performance over non-N-free KBs, a
more powerful module extractor might enable the application of
the combined \modex{}+\opt method to non-N-free \DLN knowledge bases, by
removing all normality concepts from \KB before \opt is applied.

We are also planning to adopt a different module extractor
\cite{MaWa-DL14} that is promising to be faster than the OWLAPI
implementation.

Last but not least, we are progressively extending the set of
experiments by covering the missing cases and by widening the
benchmark set, using real ontologies different from GO as well as
thoroughly synthetic ontologies.

\vspace{10pt}
\noindent
Acknowledgements: 
The authors would like to thank the reviewers for their valuable comments and suggestions.
This work has been partially supported by the PRIN project Security Horizons.

\small

\bibliographystyle{abbrv}
\bibliography{nonmon-dl,modul}

\begin{thebibliography}{10}

\bibitem{DBLP:journals/jar/BaaderH95}
F.~Baader and B.~Hollunder.
\newblock Embedding defaults into terminological knowledge representation
  formalisms.
\newblock {\em J. Autom. Reasoning}, 14(1):149--180, 1995.

\bibitem{DBLP:journals/jar/BaaderH95a}
F.~Baader and B.~Hollunder.
\newblock Priorities on defaults with prerequisites, and their application in
  treating specificity in terminological default logic.
\newblock {\em J. Autom. Reasoning}, 15(1):41--68, 1995.

\bibitem{DBLP:conf/semweb/BonattiFS10}
P.~A. Bonatti, M.~Faella, and L.~Sauro.
\newblock {EL} with default attributes and overriding.
\newblock In {\em Int. Semantic Web Conf. (ISWC 2010)}, volume 6496 of {\em
  LNCS}, pages 64--79. Springer, 2010.

\bibitem{DBLP:conf/aaai/BonattiFS11}
P.~A. Bonatti, M.~Faella, and L.~Sauro.
\newblock Adding default attributes to {EL}++.
\newblock In W.~Burgard and D.~Roth, editors, {\em AAAI}. AAAI Press, 2011.

\bibitem{DBLP:journals/jair/BonattiFS11}
P.~A. Bonatti, M.~Faella, and L.~Sauro.
\newblock Defeasible inclusions in low-complexity {DL}s.
\newblock {\em J. Artif. Intell. Res. (JAIR)}, 42:719--764, 2011.

\bibitem{DBLP:journals/jair/BonattiLW09}
P.~A. Bonatti, C.~Lutz, and F.~Wolter.
\newblock The complexity of circumscription in {DL}s.
\newblock {\em J. Artif. Intell. Res. (JAIR)}, 35:717--773, 2009.

\bibitem{DLN-15}
P.~A. Bonatti, I.~M. Petrova, and L.~Sauro.
\newblock A new semantics for overriding in description logics.
\newblock {\em Artificial Intelligence}, 222:1--48, 2015.
\newblock Available online:
  \url{http://www.sciencedirect.com/science/article/pii/S0004370215000028}.

\bibitem{DBLP:conf/jelia/CasiniMMN14}
G.~Casini, T.~Meyer, K.~Moodley, and R.~Nortje.
\newblock Relevant closure: {A} new form of defeasible reasoning for
  description logics.
\newblock In E.~Ferm{\'{e}} and J.~Leite, editors, {\em Logics in Artificial
  Intelligence - 14th European Conference, {JELIA} 2014, Funchal, Madeira,
  Portugal, September 24-26, 2014. Proceedings}, volume 8761 of {\em Lecture
  Notes in Computer Science}, pages 92--106. Springer, 2014.

\bibitem{DBLP:conf/dlog/CasiniMMV13}
G.~Casini, T.~Meyer, K.~Moodley, and I.~J. Varzinczak.
\newblock Towards practical defeasible reasoning for description logics.
\newblock In T.~Eiter, B.~Glimm, Y.~Kazakov, and M.~Kr{\"o}tzsch, editors, {\em
  Description Logics}, volume 1014 of {\em CEUR Workshop Proceedings}, pages
  587--599. CEUR-WS.org, 2013.

\bibitem{DBLP:conf/jelia/CasiniS10}
G.~Casini and U.~Straccia.
\newblock Rational closure for defeasible description logics.
\newblock In T.~Janhunen and I.~Niemel{\"a}, editors, {\em JELIA}, volume 6341
  of {\em Lecture Notes in Computer Science}, pages 77--90. Springer, 2010.

\bibitem{DBLP:journals/jair/CasiniS13}
G.~Casini and U.~Straccia.
\newblock Defeasible inheritance-based description logics.
\newblock {\em J. Artif. Intell. Res. (JAIR)}, 48:415--473, 2013.

\bibitem{DBLP:journals/tocl/DoniniNR02}
F.~M. Donini, D.~Nardi, and R.~Rosati.
\newblock Description logics of minimal knowledge and negation as failure.
\newblock {\em ACM Trans. Comput. Log.}, 3(2):177--225, 2002.

\bibitem{DBLP:series/lncs/DrabentEIKLM09}
W.~Drabent, T.~Eiter, G.~Ianni, T.~Krennwallner, T.~Lukasiewicz, and
  J.~Maluszynski.
\newblock Hybrid reasoning with rules and ontologies.
\newblock In F.~Bry and J.~Maluszynski, editors, {\em Semantic Techniques for
  the Web, The {REWERSE} Perspective}, volume 5500 of {\em Lecture Notes in
  Computer Science}, pages 1--49. Springer, 2009.

\bibitem{BaaCa10}
B.~F., C.~D., M.~D., D.~NARDI, and P.~S. P.
\newblock The description logic handbook, theory, implementation, and
  applications (2nd edition).
\newblock In {\em The Description Logic Handbook}, pages 555--555. Cambridge
  University Press, CAMBRIDGE, 2010.

\bibitem{GhilardiLutzWolter-KR06}
S.~{Ghilardi}, C.~{Lutz}, and F.~{Wolter}.
\newblock Did {I} damage my ontology? a case for conservative extensions in
  description logics.
\newblock In P.~{Doherty}, J.~{Mylopoulos}, and C.~{Welty}, editors, {\em
  Proceedings of the Tenth International Conference on Principles of Knowledge
  Representation and Reasoning (KR'06)}, pages 187--197. AAAI Press, 2006.

\bibitem{DBLP:journals/ai/GiordanoGOP13}
L.~Giordano, V.~Gliozzi, N.~Olivetti, and G.~L. Pozzato.
\newblock A non-monotonic description logic for reasoning about typicality.
\newblock {\em Artif. Intell.}, 195:165--202, 2013.

\bibitem{DBLP:journals/fuin/GiordanoOGP09}
L.~Giordano, N.~Olivetti, V.~Gliozzi, and G.~L. Pozzato.
\newblock {ALC + T}: a preferential extension of description logics.
\newblock {\em Fundam. Inform.}, 96(3):341--372, 2009.

\bibitem{DBLP:journals/jair/GrauHKS08}
B.~C. Grau, I.~Horrocks, Y.~Kazakov, and U.~Sattler.
\newblock Modular reuse of ontologies: Theory and practice.
\newblock {\em J. Artif. Intell. Res. {(JAIR)}}, 31:273--318, 2008.

\bibitem{DBLP:conf/dlog/KazakovK13}
Y.~Kazakov and P.~Klinov.
\newblock Incremental reasoning in {EL+} without bookkeeping.
\newblock In {\em Description Logics}, pages 294--315, 2013.

\bibitem{DBLP:conf/kr/KontchakovWZ08}
R.~Kontchakov, F.~Wolter, and M.~Zakharyaschev.
\newblock Can you tell the difference between dl-lite ontologies?
\newblock In G.~Brewka and J.~Lang, editors, {\em Principles of Knowledge
  Representation and Reasoning: Proceedings of the Eleventh International
  Conference, {KR} 2008, Sydney, Australia, September 16-19, 2008}, pages
  285--295. {AAAI} Press, 2008.

\bibitem{DBLP:journals/ai/Lukasiewicz08}
T.~Lukasiewicz.
\newblock Expressive probabilistic description logics.
\newblock {\em Artif. Intell.}, 172(6-7):852--883, 2008.

\bibitem{DBLP:conf/ijcai/LutzWW07}
C.~Lutz, D.~Walther, and F.~Wolter.
\newblock Conservative extensions in expressive description logics.
\newblock In M.~M. Veloso, editor, {\em {IJCAI} 2007, Proceedings of the 20th
  International Joint Conference on Artificial Intelligence, Hyderabad, India,
  January 6-12, 2007}, pages 453--458, 2007.

\bibitem{MaWa-DL14}
F.~{Martin-Recuerda} and D.~{Walther}.
\newblock Axiom dependency hypergraphs for fast modularisation and atomic
  decomposition.
\newblock In M.~{Bienvenu}, M.~{Ortiz}, R.~{Rosati}, and M.~{Simkus}, editors,
  {\em Proceedings of the 27th International Workshop on Description Logics
  ({DL'14})}, volume 1193 of {\em CEUR Workshop Proceedings}, pages 299--310,
  2014.

\bibitem{DBLP:conf/psb/Rector04}
A.~L. Rector.
\newblock Defaults, context, and knowledge: Alternatives for {OWL}-indexed
  knowledge bases.
\newblock In {\em Pacific Symposium on Biocomputing}, pages 226--237. World
  Scientific, 2004.

\bibitem{DBLP:journals/ai/Sandewall10}
E.~Sandewall.
\newblock Defeasible inheritance with doubt index and its axiomatic
  characterization.
\newblock {\em Artif. Intell.}, 174(18):1431--1459, 2010.

\bibitem{DBLP:conf/dlog/SattlerSZ09}
U.~Sattler, T.~Schneider, and M.~Zakharyaschev.
\newblock Which kind of module should {I} extract?
\newblock In B.~C. Grau, I.~Horrocks, B.~Motik, and U.~Sattler, editors, {\em
  Proceedings of the 22nd International Workshop on Description Logics {(DL}
  2009), Oxford, UK, July 27-30, 2009}, volume 477 of {\em {CEUR} Workshop
  Proceedings}. CEUR-WS.org, 2009.

\bibitem{DBLP:journals/ijmms/StevensAWSDHR07}
R.~Stevens, M.~E. Aranguren, K.~Wolstencroft, U.~Sattler, N.~Drummond,
  M.~Horridge, and A.~L. Rector.
\newblock Using {OWL} to model biological knowledge.
\newblock {\em International Journal of Man-Machine Studies}, 65(7):583--594,
  2007.

\bibitem{DBLP:journals/csec/WooL93}
T.~Y.~C. Woo and S.~S. Lam.
\newblock Authorizations in distributed systems: {A} new approach.
\newblock {\em Journal of Computer Security}, 2(2-3):107--136, 1993.

\end{thebibliography}

\newpage

\end{document}